\documentclass{article}
\usepackage{spconf,amsmath,graphicx,amsthm,amssymb,bm}
\usepackage{gensymb}
\usepackage{algorithmic}
\usepackage{algorithm}

\def\x{{\mathbf x}}
\def\y{{\mathbf y}}

\def\b{{\mathbf b}}

\def\W{{\mathbf W}}
\def\M{{\mathbf M}}
\def\V{{\mathbf V}}

\def\Tr{{\rm Tr}}
\newcommand{\R}{\mathbb{R}}

\newtheorem*{thm}{Theorem}
\title{A Spiking Neural Network with Local Learning Rules Derived From Nonnegative Similarity Matching}
\name{Cengiz Pehlevan}
\address{John A. Paulson School of Engineering and Applied Sciences, \\Harvard University, Cambridge, MA 02138, USA}
\begin{document}
%
\maketitle
\begin{abstract}
The design and analysis of spiking neural network algorithms will be accelerated by the advent of new theoretical approaches. In an attempt at such approach, we provide a principled derivation of a spiking algorithm for unsupervised learning, starting from the nonnegative similarity matching cost function. The resulting network consists of integrate-and-fire units and exhibits local learning rules, making it biologically plausible and also suitable for neuromorphic hardware. We show in simulations that the algorithm can perform sparse feature extraction and manifold learning, two tasks which can be formulated as nonnegative similarity matching problems.
\end{abstract}
\section{Introduction}

While our brains serve as an evidence of the capability and the efficiency of spike-based computation, despite the recent progress in neural networks research \cite{lecun2015deep}, spiking neural networks (SNNs) \cite{schuman2017survey} still remain largely unexplored and underutilized. This is partly because SNNs pose new challenges in theoretical understanding compared to neural networks with analogue units (AUNNs). Novel analytical methods would advance the design and analysis of SNN algorithms. 

In this paper, we present a theoretically principled approach to designing SNN algorithms that learn representations from data in an unsupervised manner. Differing from modern deep learning methods \cite{lecun2015deep}, we not only ``derive" the learning rules but also the dynamics and the architecture of an SNN from a cost function, whose optimization describes a learning task. This approach results in efficient SNNs specialized to solve the task. Access to a cost function allows prediction of the SNN algorithm's behavior on different datasets.

A challenge to the derivation of SNNs is the desired locality of an SNN's learning rules, i.e. SNN synaptic updates should depend only on the activities of the pre- and post-synaptic units. Local updates are necessary for implementability on neuromorphic hardware \cite{davies2018loihi} and also for biological plausibility. But, how could such uninformed learning be optimal in any sense? Naively, if the synapse had access to the activity of other units, it could make a better update. Indeed, existing derivations of learning SNNs ended up with non-local learning rules, starting from various cost functions and optimizing by gradient methods. To arrive at local learning, authors either resorted to approximations \cite{vertechi2014unsupervised,gilra2017predicting,alemi2018learning}, or to a contrastive learning procedure \cite{lin2018dictionary}. 

To solve the locality problem, we will adopt a technology that was recently developed for deriving AUNNs: similarity-based cost functions \cite{pehlevan2015MDS}. This family of costs lead to AUNNs with local learning rules when optimized by gradient methods  \cite{pehlevan2015MDS,pehlevan2018similarity}. We will focus on a particular similarity-based cost function, nonnegative similarity matching (NSM) \cite{kuang2012symmetric,pehlevan2014NMF}, because of its versatile usability across many learning tasks. 

In order to introduce the NSM problem, we assume the input data be a set of vectors, $\x_{t=1,\ldots,T}\in \R^n$ and the output be another set $\y_{t=1,\ldots, T}\in \R^k$. Taking dot product as a similarity measure, NSM aims to learn a representation where the similarities between output vector pairs match that of the input pairs, subject to nonnegativity constraints and regularization: 
\begin{align}\label{rNSM}
    \min_{\forall \y_t \geq 0}\, &\frac 1{2T^2}\sum_{t=1}^T\sum_{t'=1}^T  \left(\x_t^\top\x_{t'}- \y_{t}^\top \y_{t'}-\alpha^2\right)^2 \nonumber \\ &+ \frac{2\lambda_1}{T} \sum_{t=1}^T \left \Vert \y_t\right\Vert_1 + \frac{\lambda_2}{T} \sum_{t=1}^T\left\Vert \y_t \right \Vert_2^2.
\end{align}
Without the regularizers ($\alpha=\lambda_1= \lambda_2 =0$), this cost function was used for clustering  \cite{kuang2012symmetric,pehlevan2014NMF,bahroun2017neural}, sparse encoding and feature extraction \cite{pehlevan2014NMF,bahroun2017online}, and blind nonnegative source separation  \cite{pehlevan2017blind}. With $\alpha$ and $\lambda_2$ turned on, it was used for manifold learning \cite{sengupta2018manifold}. We also include an $l_1$-norm regularization for the capability of increased sparsity in the output \cite{tibshirani1996regression,zou2005regularization,hu2014SMF}. Overall, an SNN that solves the NSM problem will be a versatile tool for learning representations from data.  

A second challenge to the derivation of SNNs is the spiking dynamics. Recent work showed how to design and derive spike-based optimization algorithms to solve certain optimization problems \cite{hu2012network,shapero2014optimal,boerlin2011spike,tang2017sparse} in non-learning settings. In order to derive a learning SNN with local learning rules, we will use one of these spike-based algorithms, that of \cite{tang2017sparse}, to optimize the NSM cost function \eqref{rNSM}. We call the resulting algorithm Spiking NSM.

The rest of the paper is organized as follows. In Section \ref{AUNN}, we show how an AUNN with local learning rules can be derived from NSM. Building on these results, in Section \ref{SNN}, we derive the Spiking NSM algorithm. In Section \ref{Numerics}, we present our numerical experiments on sparse encoding and feature extraction, and manifold learning. We conclude in Section \ref{DC}.

\section{Derivation of an AUNN algorithm from the NSM cost function}\label{AUNN}

We start by deriving an AUNN from the cost in \eqref{rNSM}. While our presentation mostly follows previous accounts \cite{pehlevan2014NMF,pehlevan2018similarity}, there are important variations and novelties that enable the derivation of the Spiking NSM algorithm in the next section.

There are immediate problems with deriving an AUNN from \eqref{rNSM}. It has only inputs and outputs, but not synaptic weights. A neural network operates in an online fashion, producing an output $\y_t$ immediately after seeing an input $\x_t$, but, in \eqref{rNSM}, pairs of inputs and outputs from different time points interact with each other.

These problems can be solved by following the procedure described in \cite{pehlevan2018similarity}. Starting from the NSM cost \eqref{rNSM}, we obtain a dual min-max objective by introducing new auxiliary variables $\W\in\R^{k\times n}$, $\M \in \R^{k\times k}$, and $\b\in \R^k$, which will be interpreted as synaptic weights shortly:
%
\begin{align}\label{rNSM_stoch}
    \min_{\W\in\R^{k\times n}} \max_{\M \in \R^{k\times k}} \max_{\b\in \R^k}\frac 1T \sum_{t=1}^T l_t(\W,\M,\b),
\end{align}
where
\begin{align}
    l_t &:= \Tr \, \W^\top \W - \frac 12 \Tr \, \M^\top \M -  \left\Vert \b\right\Vert_2^2 \nonumber \\ &\qquad\qquad\qquad\qquad\qquad\qquad + \min_{\y_t \geq 0} h_t(\W,\M,\b,\y_t), \nonumber \\
    h_t &:= -2\y_t^\top \left(\W \x_t-\alpha\b\right) + \y_t^\top \M \y_t \nonumber \\
    &\qquad \qquad\qquad\qquad\quad\quad + 2\lambda_1 \left \Vert \y_t\right\Vert_1+ \lambda_2\left\Vert \y_t \right \Vert_2^2.
\end{align}
The new objective \eqref{rNSM_stoch} is equivalent to \eqref{rNSM} upto a change in the order of optimization, which can be seen by plugging back the optimal values of $\W^* = \frac 1T \sum_{t}\y_t \x_t^\top$, $\M^* = \frac 1T \sum_{t}\y_t\y_t^\top$ and $\b^*=\frac{\alpha}{T}\sum_t \y_t$. 

The min-max objective allows an online NSM algorithm because the objective is factorized into a summation of terms, $l_t$, in a way that pairs of inputs and outputs from different time points are decoupled. For each input ${\bf x}_t$, we use a two-step alternating optimization procedure \cite{olshausen1996emergence,arora2015simple} on $l_t$ to produce an output ${\bf y}_{t}$ and update variables $\W$, $\M$, and $\b$. We now discuss these steps, and how they map to the operations of an AUNN with local learning rules.

\subsection{Solving for outputs with an AUNN}

The first step of the alternating optimization is minimizing $h_t$ (and $l_t$) with respect to nonnegative $\y_t$, while keeping $\W$, $\M$, and $\b$ fixed\footnote{Note that this is a nonnegative elastic net problem \cite{zou2005regularization}.}. Define $\bar \M := \M-\text{diag}\left(\M\right)$, where $\text{diag}$ operator sets off-diagonal elements of a matrix to zero. Then, the following dynamical system minimizes $h_t$\footnote{It is easy to show that for this dynamics $\frac{d h_t}{d\tau} \leq 0$. A rigorous discussion of convergence can be found in \cite{tang2016convergence} (see the Conclusion section of \cite{tang2016convergence}).}:
\begin{align}\label{rDyn}
   \frac{d u_{i}(\tau)}{d\tau} &= -u_{i}(\tau) + \left[\W\x_t\right]_i - \alpha b_i -\left[\bar \M \y_t(\tau)\right]_i,\nonumber \\
    y_{t,i}(\tau) &= g_i(u_{i}(\tau)) := \left\lbrace \begin{array}{ll}
       0,  & u_{i}(\tau) \leq \lambda_1 \\
        \frac{u_{i}(\tau)-\lambda_1}{\lambda_2 + M_{ii}}, & u_{i}(\tau) > \lambda_1
    \end{array}\right., \nonumber \\ 
    &\;\qquad\qquad\qquad\qquad\qquad\qquad i=1,\ldots,k.
\end{align}
This system can be interpreted as the dynamics of the neural network shown in Figure \ref{fig:net}. $\x_t$ is the input to the network and $\y_t$ is the output vector of unit activities. $\W$ and $-\bar \M$ are feedforward and lateral synaptic weight matrices. $-\alpha b_i$ is the synaptic weight to unit $i$ from an input unit with activity $1$. Finally, $g_i$ is a unit-dependent activation function. 

We note that previous AUNN derivations from similarity-based cost functions used subgradient descent \cite{hu2014SMF} , projected gradient descent \cite{pehlevan2018similarity}
or coordinate descent \cite{pehlevan2014NMF,pehlevan2015MDS,seung2017correlation} dynamics for this step of the algorithm. Our dynamics choice here is motivated by its generalization to an SNN, which will be presented in the Section \ref{SNN}.
\begin{figure}
\centering
\includegraphics[scale = 1]{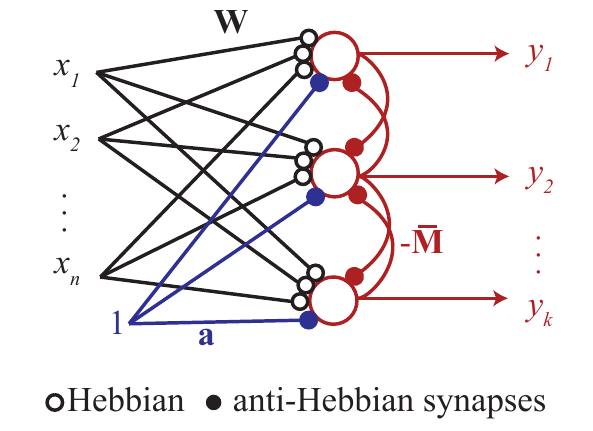}
\caption{The network architecture. Some lateral connections are not shown for better visibility. $a_i = -\alpha b_i$ for the AUNN algorithm, and $a_i = -\alpha b_i-\lambda_1$ for the SNN algorithm.}
\label{fig:net}
\end{figure}

\subsection{Updating synaptic weights with local learning rules}
The second step of the alternating optimization is to perform gradient updates in $\W$, $\b$, and $\M$ with fixed ${\bf y}_t$, which we write in component notation to expose their locality:
\begin{align}\label{syn_update}
    \Delta W_{ij} &= \eta \left(y_{t,i} x_{t,j}-W_{ij}\right), \;
    \Delta \bar M_{ij, i\neq j} = \eta\left(y_{t,i} y_{t,j}-\bar M_{ij}\right),  \nonumber \\
     \Delta M_{ii} &= \eta\left(y_{t,i}^2- M_{ii}\right), \;\;\;\;\;\;
    \Delta b_i = \eta\left(\alpha y_{t,i}-b_i \right),
\end{align}
where $\eta$ is a learning rate. $\W$ update is a Hebbian synaptic plasticity rule. $\bar \M$ and $\b$ updates are anti-Hebbian (because of the $-$ signs in corresponding terms in \eqref{rDyn}). $M_{ii}$ update changes the gain function of a neuron, and can be interpreted as a homeostatic plasticity rule.  

\section{Derivation of the Spiking NSM algorithm from the NSM cost function}\label{SNN}

Next, we derive an SNN algorithm with local learning rules, the Spiking NSM algorithm, from the NSM cost function \eqref{rNSM}. We do this by replacing the optimization algorithm for minimizing $h_t$ with a spike-based one. 

\subsection{Solving for outputs with an SNN}

The SNN of Tang, Lin and Davies \cite{tang2017sparse} minimizes $h_t$. We first describe the SNN and then cite a theorem for its convergence to the fixed point of the AUNN given in \eqref{rDyn}. Since the AUNN fixed point is the minimum of $h_t$, the SNN minimizes $h_t$.

Consider a network of $k$ integrate-and-fire units with the same architecture as in Figure \ref{fig:net}. We denote the $i^{\rm th}$ unit's membrane potential by $V_i(\tau)$, input current by $I_i(\tau)$, and $q^{\rm th}$ spike time by $\tau_{i,q}$. The units are perfect integrators, i.e. their subthreshold membrane potentials are given by:
\begin{align}\label{Vi}
    \frac {d V_i(\tau)}{d\tau} = I_i(\tau), \qquad V_i (0) = 0, \qquad i = 1,\ldots,k.
\end{align}
When $V_i(\tau)$ reaches a firing threshold $V^{{\rm th}}_i := \lambda_2 + M_{ii}$, the unit emits a spike and the membrane potential is set to 0. Synaptic input is defined as
\begin{align}\label{Ii}
    \frac{d I_i(\tau)}{d\tau} &= -I_i(\tau) + \left[{\bf W}{\bf x}_t\right]_i  -\alpha b_i -\lambda_1 - \left[\bar \M \bm{\sigma}(\tau)\right]_i, \nonumber \\ 
    I_i(0) &= \left[{\bf W}{\bf x}_t\right]_i  -\lambda_1-\alpha b_i, \qquad i=1,\ldots,k,
\end{align}
where $\x_t$ is the input to the network, $\sigma_i(\tau):= \sum_q \delta(\tau-\tau_{i,q})$ is the spike train of the $i^{\rm th}$ unit, $\W$ and $-\bar \M$ are feedforward and lateral synaptic weight matrices,  and $-\alpha b_i-\lambda_1$ is the synaptic weight to unit $i$ from an input unit with activity $1$.

Under mild assumptions, it can be shown that time-averaged spike trains converge to the fixed point of the AUNN defined in \eqref{rDyn} and therefore minimize $h_t$. More precisely, let's define
\begin{align}\label{yspike}
    \tilde y_{i}(\tau) := \frac{1}\tau \int_0^\tau d\tau'\, \sigma_i(\tau'), \qquad i = 1,\ldots,k.
\end{align}
%
%

\begin{thm}[Tang, Lin, Davies \cite{tang2017sparse}](Informal) Assume that the duration between a unit's consecutive spikes is not arbitrarily long but upper bounded, with the exception of the unit stopping spiking altogether after some time. Then, as $\tau\to \infty$, $\tilde y_{i}(\tau)$ converges to the value of $y_{t,i}$ at the fixed point of the dynamical system \eqref{rDyn}.
\end{thm}

\begin{proof}
The results of \cite{tang2017sparse} can be easily extended to prove this result. See especially the Discussion section of \cite{tang2017sparse}.
\end{proof}

\subsection{Updating synaptic weights with local learning rules}

After the spiking dynamics converges, we update  $\W$, $\b$ and $\M$ as in \eqref{syn_update}, using the spiking estimate for $\y_t$ from \eqref{yspike}. Note that $M_{ii}$ updates are still interpreted as homeostatic plasticity, but this time they change the firing thresholds of units. 

The final Spiking NSM algorithm is summarized below.

\begin{algorithm}[H]
  \caption{Spiking NSM}
  \begin{algorithmic}\label{alg:sNSM}
  \renewcommand{\algorithmicrequire}{\textbf{Input:}}
  \REQUIRE Parameters $\alpha$, $\lambda_1$ and $\lambda_2$. Initial weights $\M \in\R^{k\times k}$, $\W \in\R^{k\times n}$ and $\b \in \R^k$.
   \FOR {$t = 1,2,3,\dots$}
   \STATE\texttt{// Spiking neural dynamics}
   \STATE Taking $\x_t$ as input, run the SNN defined by equations \eqref{Vi}, \eqref{Ii}, \eqref{yspike} until convergence.
   \STATE\texttt{// Synaptic and homeostatic plasticity}
   \STATE Update  $\W$, $\b$ and $\M$ as in \eqref{syn_update}, using the spiking estimate for $\y_t$ from \eqref{yspike}.
   \ENDFOR
  \end{algorithmic}
\end{algorithm}

\section{Experimental results}\label{Numerics}

In this section, we apply the Spiking NSM algorithm to various datasets for sparse encoding and feature extraction \cite{pehlevan2014NMF}, and manifold learning \cite{sengupta2018manifold}. Our purpose here is not to compare the performance of NSM with other unsupervised learning methods, this was done in \cite{bahroun2017neural,bahroun2017online,pehlevan2017blind}. We wish to demonstrate that Spiking NSM actually performs online NSM.

\subsection{Solving for outputs with spike-based dynamics}

Before attempting a learning task, we first checked whether the SNN defined by equations \eqref{Vi}, \eqref{Ii}, \eqref{yspike} indeed minimizes $h_t$, by simulating it with randomly chosen values for lateral connectivity and inputs. More precisely, we set $\alpha = 0.3$, $\lambda_1 = 0.3$, $\lambda_2 = 0.1$, drew $b_i$ from a uniform distribution in $[0,1]$ and $\left(\W\x\right)_i$ in $[0,5]$, and set $\M = \V\V^\top$, where $V_{ij}$ were drawn from a uniform distribution in $[0,1/\sqrt{k}]$. For each $k \in \lbrace 2,4,8,16,32,64,128,256\rbrace$, we repeated this procedure until we obtained 100 accepted parameter sets. A parameter set was accepted if the norm of the minimum of $h_t$, found by MATLAB's fmincon function, had an $l_2$-norm greater than 0.01. For each parameter choice, we simulated the corresponding SNN until $\tau=500$, using a first-order Euler method with a step-size $d\tau = 0.01$. 

\begin{figure}[h]
\centering
\includegraphics[scale = 1]{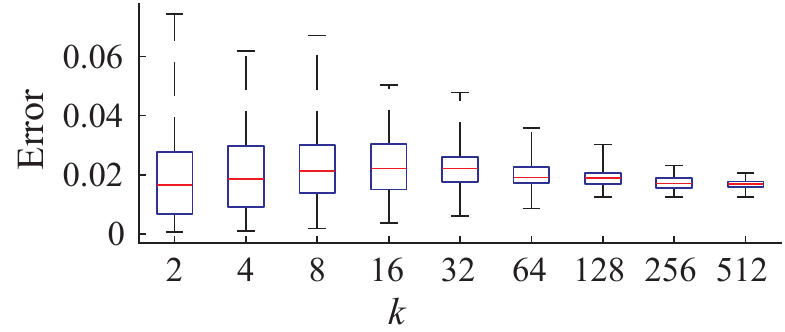}
\caption{Spike-based minimization of $h_t$. Red lines show median, box edges show 25th and 75th percentiles and whiskers show maxima and minima.}
\label{fig:spikemin}
\end{figure}

Figure \ref{fig:spikemin} shows that the SNN achieves the minimum of $h_t$ within a few percent. The error for a parameter set was measured by $\left\Vert \y - \hat \y \right\Vert_2/\left\Vert \hat \y \right\Vert_2$, where $\y$ is the result from the SNN and $\hat \y$ from fmincon. We observed that longer simulation times and finer step-sizes led to better performance. In the rest of this section, we used the simulation time and step size configuration used here.

\subsection{Sparse encoding and feature extraction}
Previously, NSM was shown to extract sparse features from data \cite{pehlevan2014NMF,pehlevan2017blind}. We trained Spiking NSM networks on two datasets to test this function: 1) the MNIST dataset of $6\times 10^4$ images of hand-written digits \cite{lecun1998mnist}, and 2) a dataset of $4\times 10^5$ 16-by-16 image patches sampled randomly from a set of whitened natural scenes \cite{olshausen1996emergence}. For both simulations, $\alpha = \lambda_2 = 0$, initial $\M$ was set to identity matrix and $\b$ to zero.  Learning rates were $10^{-3}$ for the first $10^4$ steps, $10^{-5}$ for the next $9\times10^4$ steps, and $0.5\times 10^{-5}$ later. For MNIST, $\lambda_1=0.5$, $k=196$ and initial $W_{ij}$ were drawn uniformly from $[0,1/14]$. For image patches, $\lambda_1=0$, $k=256$ and initial $W_{ij}$ were drawn from $\mathcal{N}(0,1/196)$. At each iteration, a randomly chosen datum was shown to the network. 

\begin{figure}[h]
\begin{minipage}[b]{.48\linewidth}
  \centering
  \centerline{MNIST}\medskip
  \centerline{\includegraphics[width=4.1cm]{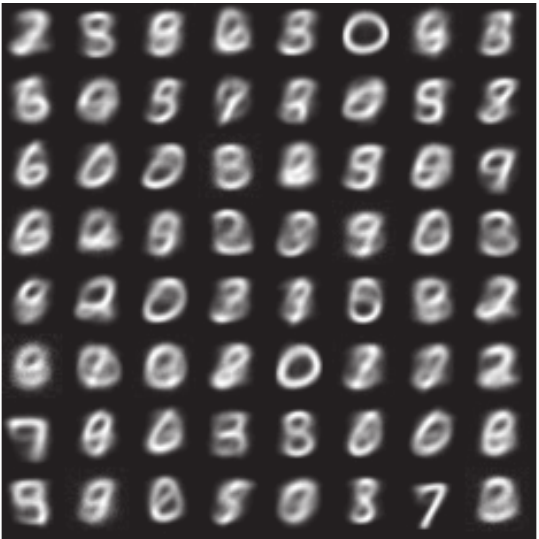}}
\end{minipage}
\hfill
\begin{minipage}[b]{0.48\linewidth}
  \centering
  \centerline{Image patches}\medskip
  \centerline{\includegraphics[width=4.1cm]{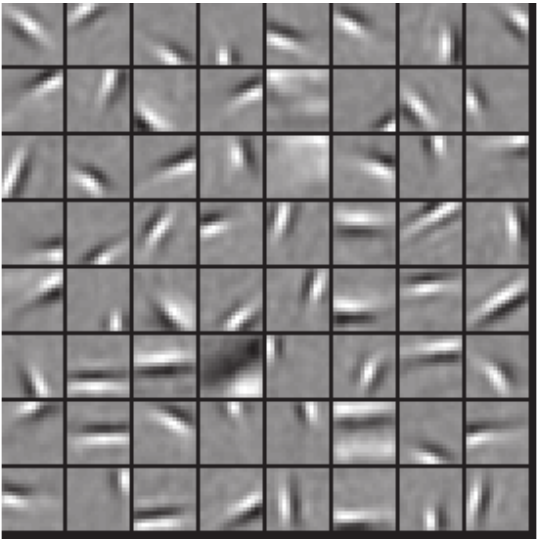}}
\end{minipage}
\centerline{(a) Examples of learned features}
\vskip 10pt
\begin{minipage}[b]{.48\linewidth}
  \centering
  \centerline{\includegraphics[width=4.1cm]{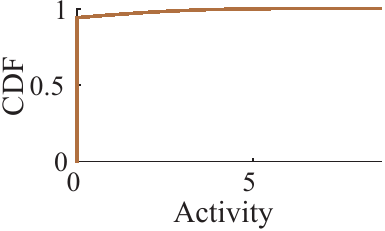}}
\end{minipage}
\hfill
\begin{minipage}[b]{0.48\linewidth}
  \centering
  \centerline{\includegraphics[width=4.1cm]{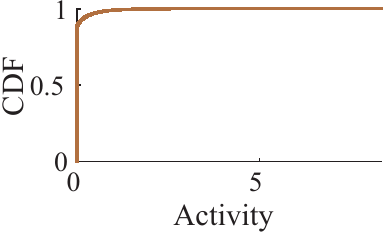}}
\end{minipage}
\centerline{(b) CDFs of output unit activity }
\caption{Spiking NSM extracts sparse features.}
\label{fig:MNIST}
\end{figure}

Figure \ref{fig:MNIST} shows the results of our simulations after $10^6$ iterations. Panel (a) displays 64 examples of learned features for each dataset, extracted from the rows of $\W$ as in \cite{pehlevan2014NMF}. Image patch features are oriented edges, as in sparse coding \cite{olshausen1996emergence}. Panel (b) displays cumulative distribution functions (CDFs) of the networks' time-averaged spiking output activities (Eq. \eqref{yspike}) calculated over the whole datasets. Activities are highly sparse.

\subsection{Manifold learning}
NSM learns a data manifold by learning features that tile the manifold \cite{sengupta2018manifold}. To test this function, we used a one-dimensional data manifold in a high dimensional space, composed of 71 576-by-768 images of a shoe rotated by 5\degree increments \cite{geusebroek2005amsterdam} (examples shown in Figure \ref{fig:shoe}(a)). After normalizing each image to unit norm, we trained a Spiking NSM with $k=100$ output units and $\lambda_1 = \lambda_2 = 0$, $\alpha = 0.8$. The $\alpha$ parameter sets the scale of local similarity neighborhoods \cite{sengupta2018manifold}. The rest of the simulation parameters matched the MNIST case.

\begin{figure}[h]
\begin{minipage}[b]{.48\linewidth}
  \centering
  \centerline{\includegraphics[width=4.1cm]{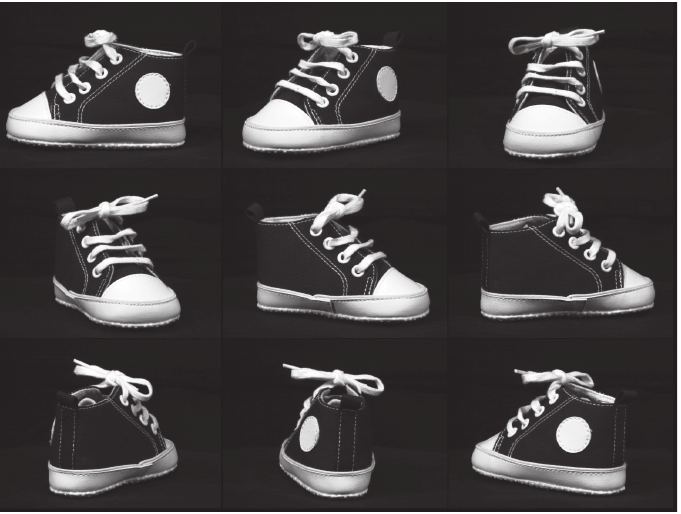}}
  \centerline{(a) Examples of input images}\medskip
\end{minipage}
\hfill
\begin{minipage}[b]{0.48\linewidth}
  \centering
  \centerline{\includegraphics[width=4.1cm]{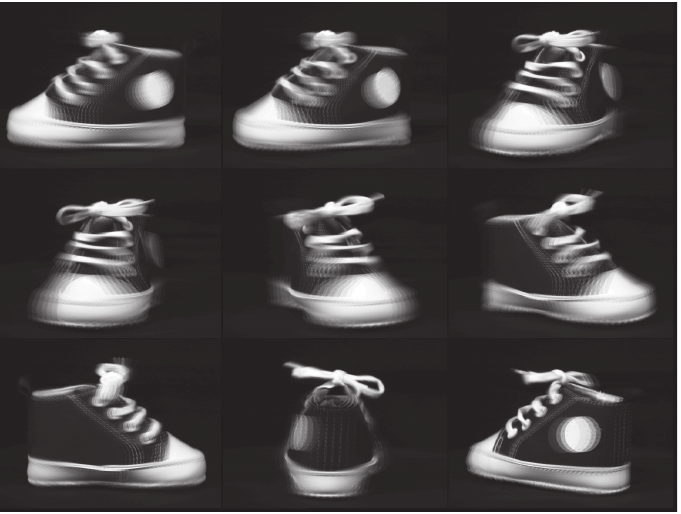}}
  \centerline{(b) Examples of learned features}\medskip
\end{minipage}
\begin{minipage}[b]{\linewidth}
  \centering
  \centerline{\includegraphics[width=8.0cm]{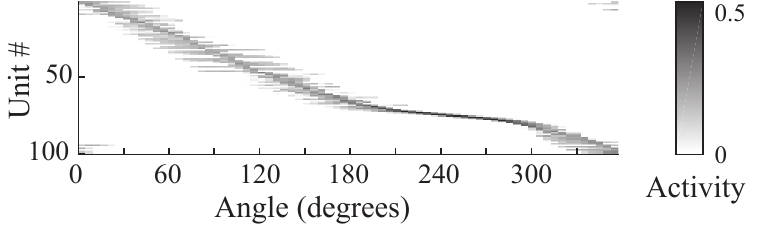}}
  \centerline{(c) Ordered unit responses as a function of the rotation angle}\medskip
\end{minipage}
\caption{Spiking NSM tiles data manifolds}
\label{fig:shoe}
%
\end{figure}

Figure \ref{fig:shoe} shows our results after $1.8\times10^4$ iterations.  Panel (b) displays 6 example learned features, which are localized to the vicinity of particular shoe rotation angles.  Panel (c) shows the time-averaged spiking activity of output units (Eq. \eqref{yspike}) as a function of the shoe's rotation angle. As promised, the units tile the data manifold.

\section{Conclusion}\label{DC}

We presented a principled derivation of the Spiking NSM algorithm, which exhibits local learning rules, from the NSM cost function \eqref{rNSM}. We applied the algorithm to various datasets and interpreted its action as sparse feature extraction and encoding, or manifold learning,  based on analytical analyses of the NSM cost function \cite{pehlevan2017blind,sengupta2018manifold}. 

With the advent of new spiking neuromorphic hardware \cite{davies2018loihi}, the need for new SNN algorithms with local learning rules is increasing. We expect the principled approach presented in this paper to be useful in this quest.


\end{document}